\documentclass[letterpaper]{article}
   \usepackage[submission]{aaai24} 
   \usepackage{times} 
   \usepackage{helvet} 
   \usepackage{courier} 
   \usepackage[hyphens]{url} 
   \usepackage{graphicx} 
   \usepackage{graphicx}  
   \usepackage{natbib}  
   \usepackage{caption}  
\usepackage[T1]{fontenc}
\usepackage{mathptmx} 
\usepackage{amsmath}
\usepackage{amssymb}
\usepackage{amsthm}
\usepackage{amsfonts}
\usepackage{color}
\usepackage{natbib}
\usepackage{algpseudocode}
\usepackage{algorithm}
\newtheorem{theorem}{Theorem}
\newtheorem{definition}{Definition}

\newtheorem{lemma}{Lemma}

\usepackage{enumitem} 
\newcommand{\abs}[1]{\left|#1\right|}
\newcommand{\E}[1]{\mathbb{E}\left[#1\right]}
\newcommand{\s}{G}
\newcommand{\valhat}{\hat{v}_i(S)}
\newcommand{\val}{v_i(S)}
\newcommand{\svkho}{\overline{\svk}}
\newcommand{\ldiff}{\ell(w_T;z)-\ell(w_T';z)}
\newcommand{\brackets}[1]{\left(#1\right)}
\newcommand{\A}{\mathcal{A}}

\newcommand{\svkh}[0]{\hat{\phi}^k_i}
\newcommand{\svk}[0]{\phi^k_i}
\newcommand{\sv}{\varphi}

\newcommand{\varplace}{\hat{\sigma}^2}

\newcommand{\vsample}{h}
\newcommand{\bound}{\epsilon^2 + \epsilon\s}

\def\sv{\varphi}
\def\dst{\displaystyle}

\newcommand{\bigL}{\mathcal{L}}
\newcommand{\norm}[1]{\lVert#1\rVert}

\newtheorem{corollary}{Corollary}[theorem]
\newtheorem*{remark}{Theorem}

\newcommand{\myparagraph}[1]{\vspace*{1mm}\noindent{\bf #1} }

\algdef{SE}[REPEATN]{RepeatN}{End}[1]{\algorithmicrepeat\ #1 \textbf{times}}{\algorithmicend}

\title{Accelerated Shapley Value Approximation for Data Evaluation}

 \author{Lauren Watson,\textsuperscript{1}
Zeno Kujawa,\textsuperscript{2}
Rayna Andreeva,\textsuperscript{1}
Hao-Tsung Yang ,\textsuperscript{3}
Tariq Elahi ,\textsuperscript{1}
Rik Sarkar\textsuperscript{1}\\
\textsuperscript{1}{School of Informatics, University of Edinburgh}\\
\textsuperscript{2}{University of Cambridge}\\
\textsuperscript{3}{National Central University }\\
lauren.watson@ed.ac.uk, rsarkar@inf.ed.ac.uk}

\begin{document}
\maketitle
\begin{abstract}
Data valuation has found various applications in machine learning, such as data filtering, efficient learning and incentives for data sharing. The most popular current approach to data valuation is the Shapley value. While popular for its various applications, Shapley value is computationally expensive even to approximate, as it requires repeated iterations of training models on different subsets of data. In this paper we show that the Shapley value of data points can be approximated more efficiently by leveraging the  structural properties of machine learning problems. We derive convergence guarantees on the accuracy of the approximate Shapley value for different learning settings including Stochastic Gradient Descent with convex and non-convex loss functions. Our analysis suggests that in fact models trained on small subsets are more important in the context of data valuation. Based on this idea, we describe $\delta$-Shapley -- a strategy of only using small subsets for the approximation. Experiments show that this approach preserves approximate value and rank of data, while achieving speedup of up to $9.9x$. In pre-trained networks the approach is found to bring more efficiency in terms of accurate evaluation using small subsets.
\end{abstract}
\section{Introduction}

Data valuation is a fundamental research problem in machine learning. With widespread use and monetization of personal data in AI and machine learning models, the value of data has become a concern in data related governance and economics, since policies, incentives and compensations depend on fair valuation~\cite{posner2018radical}. In the model training process itself, valuations can be used to select useful data items~\cite{pooladzandi2022adaptive}, efficient  active learning~\cite{ghorbani2022data},  federated learning~\cite{wang2020principled} or explainability~\cite{koh2017understanding}.

The current popular approaches in data valuation~\cite{ghorbani_data_2019, jia2019towards} are based on the Shapley value. Shapley value~\cite{shapley1953value} is a concept from cooperative game theory for fair measure of the contributions of players in a game. In data valuation, the Shapley value of each data point measures its contribution to the utility of a specific model. Utility can be measured based on the training or test set loss. The contribution of a data point to that utility is  harder to quantify as it depends on the model as well as the rest of the data. Shapley value is a solution to this problem as it  incorporates these complexities and is known to produce equitable valuations~\citep{ghorbani_data_2019}.

The usefulness and wide applicability of Shapley values comes at the cost of computational time and resources. Computation of the Shapley value of a data point involves computing the weighted average of the point's marginal contribution to all possible data subsets (called {\em coalitions} in game theory). The direct computation is naturally exponential time and the problem is computationally hard \cite{deng_papadimitriou, nagamochi}. The approximation of Shapley value involves Monte Carlo sampling from all possible subsets and evaluating them by building models. While this is more efficient than exact computation, it still requires evaluating $\Omega(n)$ coalitions  \cite{maleki_bounding_2014}, and each evaluation involves training and testing a model on a large dataset. In this paper, we show that the specific structure of machine learning problems implies that a more efficient approximation is possible for data valuation.

\myparagraph{Our Contributions.} In this paper we present more efficient  methods for approximating the Shapley values of data points for models trained in both  convex and non-convex scenarios. These methods produce values close to the Monte Carlo sampling algorithm, but run several times faster. The key to the algorithmic speedup is to rely on  smaller coalitions where model training is efficient, while foregoing larger coalitions and expensive model training. We also define a valuation called the $\delta-$Shapley that exclusively uses small coalitions and in practice performs as well as the Shapley value.
The main insight in our analysis is that the marginal contribution of a single data point is limited when the coalition is large. This property can be derived based on the stability properties of machine learning algorithms. The precise properties and results vary across learning scenarios.

Strongly convex loss functions are known to satisfy uniform stability~\cite{bousquet_stability_2002}, where the difference in loss due to the presence or absence of a single data point can be bounded by $O(\frac{1}{n})$ for datasets of size $n$. We show that this property implies that the effect of large coalitions on the Shapley value of a data point is limited, and can be approximated efficiently. In fact, to estimate the effect of all coalitions of size $k$, about $O(\frac{1}{k^2})$ sample coalitions suffice to obtain an error bound of $a$ with probability at most $b$.

When the model is trained via a randomized mechanism such as Stochastic Gradient Descent on non-convex loss, the strict stability bound no longer holds. In this case, a suitable property defined in~\citet{hardt_train_2016} which we call {\em Expected Uniform Stability} applies. We provide and discuss two different interpretations of Shapley value for randomized algorithms and derive $(a,b)$ bounds for convex and non-convex losses.

Inspired by the approximation idea above, we define a new value function called the $\delta-$Shapley. This value is characterized by two observations: first, that larger coalitions have little impact on valuation, as seen in our analysis, and second, that tiny coalitions are likely to produce poorly trained models whose contributions are largely noise. Thus, $\delta-$Shapley uses coalitions from a small band of sizes in the small to medium range. Similar to Shapley and other values in the class, $\delta-$Shapley is a semi-value~\cite{dubey1981value}.

Experiments show that $\delta-$Shapley values produce meaningful data valuations by accurately evaluating contributions of data points to test accuracy across multiple classification tasks.  $\delta-$Shapley values are also highly correlated  with Monte Carlo estimates of Shapley values ($\rho\in[0.84,0.98]$) for convex and non-convex classification tasks. We find that sampling mid-sized or small coalitions produces high correlations and that both methods significantly reduce computation time in comparison to the Monte Carlo approach, with a greater reduction in time when sampling from lower coalition sizes. Fine tuning pre-trained neural networks is a common approach to model development. We find that on pre-trained networks, $\delta-$Shapley is particularly effective, and can obtain accurate estimates even with small coalitions.

\section{Technical Background}
In this section, we review the formal definition  of Shapley value, its use in data valuation,  relevant concepts in machine learning and Stochastic Gradient Descent (SGD).

\subsection{The Shapley value}

Shapley value is defined in terms of a utility function $v$, where $v(S)$ represents the utility generated by a set $S$ of players. The marginal utility of a player $i$ with respect to $S$ is written as $v(S\cup i) - v(S)$. We will sometimes abbreviate it as $v_i(S)$. Also note that we use $S\cup i$ to mean $S\cup\{i\}$ for better readability.
\begin{definition}\label{def:SV}
The Shapley value~\cite{shapley1953value} of player $i$ in a set $D$ of $n$ players, written as $\varphi_i$, is given by \begin{equation}\label{eq:shapley}
   \varphi_i(v) = \frac{1}{n}\dst\sum_{S\subseteq D\setminus\{i\}} {n-1 \choose \abs{S}}^{-1} (v(S\cup i) - v(S)).
\end{equation}
\end{definition}

The Shapley value is the unique single valued solution satisfying the following four natural axioms~\cite{shapley1953value}:
\begin{itemize}
    \item {\bf Dummy (Null) Player}: if $v(S\cup i)=v(S)+c$  for all $S\subseteq D\setminus\{i\}$ and some $c\in\mathbb{R}$, then $\varphi_i=c$.
    \item {\bf Symmetry}: if $v(S\cup i)=v(S\cup j)$ for all $S\subseteq D\setminus\{i,j\}$, then $\varphi_i=\varphi_j$.
    \item {\bf Linearity}: for two value functions $v_1$ and $v_2$,  $\varphi_i(a_1v_1+a_2v_2)=a_1\varphi_i(v_1)+a_2\varphi_i(v_2)$ for any $a_1, a_2\in\mathbb{R}$.

    \item {\bf Efficiency}: for every value function $v$, $\sum_{i=1}^n\varphi_i=v(D)$.
\end{itemize}

\subsubsection{Data valuation with Shapley values.} In data valuation, each data  point naturally takes the role of a player, and a subset of data points is equivalent to a coalition. The utility can be measured either in terms of loss or accuracy, with validation loss being the more popular utility~\cite{ghorbani_data_2019,jia2019towards}.
We assume there exists a test set $D_e$ of data unseen during training, and the overall test loss for the output $w=A(S)$ of algorithm $A$ is:
\[
    \bigL(w) = \frac{1}{\abs{D_e}}\sum_{z \in D_e} \ell(w, z)
\]
where $\ell(w, z)$ is the loss of the model $w$ on the data point $z$. The utility $v$ is then defined by $v(S)= - \bigL(w)$~\cite{ghorbani_data_2019}.
Thus with $w^\prime = A(S \cup {i})$ and $w = A(S)$, the marginal contribution of $i$ to $S$ is: $v_i(S) =  \bigL(w) -\bigL(w^\prime)$.

\subsubsection{Dependence on model and algorithms}
Observe that the utility, marginal utility  and therefore Shapley values are functions of the specific models found by the algorithm. While many different models can be computed on the same dataset, the data valuation question we consider here is not one of the general useability of a data point, but of how much a data point $i$ has contributed to the specific model $M$.

\subsubsection{Shapley Approximation via Monte Carlo Sampling}
In each iteration of the Monte Carlo procedure \cite{CASTRO}, a random permutation $\pi(D)$ of the dataset $D$ is sampled. The set $\text{Pre}^i(\pi(D))$ of all predecessors of $i$ forms the set $S$, and $v_i$ is then computed as described above. An average of the sampled $v_i(S)$ is taken to obtain an estimate of the Shapley Value of $i$. The process is repeated a suitable number of times~\cite{ghorbani_data_2019}. When contributions of all data points need to be computed, a  marginal contribution for all data points can be computed before the next permutation is sampled.

\subsection{Machine learning}

In machine learning, a training dataset $D$ (drawn from a distribution $Z$) is used to find a model $w$ within the model class $\mathcal{H}$. In common setups, $w$ is usually a vector of real numbers representing the model parameters. The training process aims to find a model $w$ that minimizes the training loss given by $L(w,D) = \frac{1}{\abs{D}}\sum_{z\in D}\ell(w,z)$.

\subsubsection{Stochastic Gradient Descent}

To find a model with nearly minimum loss, or sufficiently low loss, a popular algorithm is Stochastic Gradient Descent (SGD), sometimes called the Stochastic Gradient Method (SGM). SGD operates by making incremental steps that locally reduce the loss. In each step $t$, SGD uses a random $z\in D$, computes the gradient of $\ell(w_t, z)$ and makes an incremental move in the direction of the gradient:
$
    w_{t+1} = w_t - \alpha_t \nabla_w \ell(w_t;z)
$
where $\alpha_t$ is the learning rate at step $t$. In different variations of SGD, the choice of $z$ may be made independently with replacement in each step, or the choice may be determined by cycling through different permutations of $D$. Each such cycle is called an epoch of SGD.

\subsubsection{Convex and non-convex loss functions}

\begin{definition}[Convexity and Strong convexity]
    A function $f: \Omega \rightarrow \mathbb{R}$ is convex if for all $u,v \in \Omega$
    $f(u) \geq f(v) + \langle \nabla f(v),
 u-v\rangle$ and $\lambda$-strongly convex for $\lambda>0$ if for all $u,v \in \Omega$ $f(u) \geq f(v) + \langle \nabla f(v),
 u-v\rangle + \frac{\lambda}{2}\lVert u-v \rVert^2$.

\end{definition}
Any local minimum of a convex function is also a global minimum. Additionally, a strongly convex functions has a unique minimum. Optimization in convex and strongly convex cases are considered easier as methods like SGD can converge to the global minimum. Non-convex losses are harder to optimize, as SGD may converge to a local minimum and there is no guarantee of approaching the global minimum. SGD with non-convex losses is an important current problem as many popular model classes including neural networks have non-convex losses and are trained using SGD or its variants. Loss functions are frequently assumed to have certain continuity properties for theoretical analysis:

\begin{definition}[Lipschitzness]
    A function $f: \Omega \rightarrow \mathbb{R}$ is L-Lipschitz if for all $u, v \in \Omega$
    $||f(u) -f(v)|| \leq L||u-v||$.
\end{definition}

\begin{definition}[Smoothness]
    A function $f: \Omega \rightarrow \mathbb{R}$ is $\beta-$smooth if for all $u,v \in \Omega$ $||\nabla f(u) -\nabla f(v)|| \leq \beta||u-v||$.

\end{definition}

Next, we describe out general approach and analysis for various scenarios.

\section{Approximation algorithm and Strongly Convex Optimization}

To approximate the Shapley value of $i$, we divide all coalitions (subsets of $D$) into $n$ {\em layers}, where $S_k$ is layer $k$ consisting of all subsets of size $k$: $S_k = \{ S\subseteq D: \abs{S}=k\}$.

\begin{definition}[Contribution at layer $k$] The average marginal contribution of layer $k$ is:
\[
\phi_i^k = {n-1 \choose k}^{-1} \sum_{S\in S_k} v_i(S).
\]
\end{definition}

The Shapley value $\varphi_i$ as in Definition~\ref{def:SV} is then simply given by the average: $\varphi_i = \frac{1}{n}\sum_{k = 0}^{n-1}\svk$. Thus, the main computational task of Shapley value can be viewed as finding the average marginal contributions $\svk$.

This computation is expensive when ${n-1 \choose k}$ is large. We can instead estimate $\svk$ via sampling. Our strategy is to estimate each $\svk$ independently. Algorithm~\ref{alg:cap} shows the idea.

\begin{algorithm}[hbt!]
\caption{$\hat{\svk}$: Expected marginal contribution in layer $k$.}\label{alg:cap}
    \begin{algorithmic}[1]
    \State Input: Data point $i$, Dataset $D$, coalition size $k$, utility function $v$
    \State Output: $\hat{\svk}$: estimate of marginal contribution of $i$ to coalitions of size $k$
    \State Initialize $\svk=0$
    \State Compute $m_k$ \Comment{The number of samples needed at layer $k$}
    \RepeatN{$m_k$}
        \State Draw $S\in{S}_k$ at random
        \State $\hat{\svk}\leftarrow \hat{\svk} + \frac{1}{m_k} (v(S\cup i) - v(S))$ \Comment{Update estimate $\hat{\svk}$}
    \End
    \State return $\hat{\svk}$
    \end{algorithmic}
\label{algo:stratified-sampling}
\end{algorithm}

To obtain the Shapley value, the most natural method will be to compute all $\svk$'s and average them as above. Other methods for more efficient approximation are covered later.

The missing element in Algorithm~\ref{alg:cap} is the computation of sample size $m_k$. We discuss below the analysis of $m_k$ for strongly convex loss functions.

\subsection{Sample size for strongly convex optimization}

In this section, we demonstrate the core idea of our approach in the case of strongly convex optimization. Strong convexity of the loss function holds for a large class of problems such as regularized regression, logistic regression and support vector machines. For our analysis, a strongly convex loss implies an important property of the learning algorithm called uniform stability, which is defined as a small change in the underlying dataset implying only a small change in the algorithm's output.

\begin{definition}[Neighboring Datasets]
Two datasets $S, S'$ are neighboring if $H(S, S') \leq 1$, where $H(\cdot, \cdot)$ represents the hamming distance.
\end{definition}

\begin{definition}[Uniform stability \cite{bousquet_stability_2002}]\label{def:unistab} A learning algorithm $A$ has $\gamma$-uniform stability with respect to the loss function $\ell$ if the following holds,$\forall z\in Z, \forall S \in  Z^n, \forall i \in \{1, ..., n\}$,
\begin{equation*}
   \norm{\ell(A(S), z)-\ell(A(S^{\backslash  i}), z)}_{\infty}\leq \gamma.
\end{equation*}
Here $S^{\backslash  i} = S\setminus \{i\} $, so $S, S^{\backslash  i}$ are neighboring datasets.
\end{definition}
It can be shown that strong convexity implies a uniform stability of $\gamma = \frac{c}{n}$ where $n=\abs{S}$ and $c$ is constant independent of $S$. Thus, on a coalition of size $k$, strong convexity will imply a uniform stability of $\frac{c}{k}$. Note that a smaller value of uniform stability implies a more stable algorithm.

\begin{theorem}~\cite{bousquet_stability_2002}\label{thm:sc_stab}
    Suppose we have learning function $A:Z^n\rightarrow\mathcal{H}$ that finds $w\in\mathcal{H}$ minimizing $L(w,D)=\frac{1}{\abs{D}}\sum_{z\in D}\ell(w,z)$ for $D\in Z^n$. Let $\mathcal{H}$ be a reproducing kernel Hilbert space with kernel $\zeta$ such that $\forall z \in D $ we have $ \zeta(z, z) \leq C$. Then if $\ell$ is $L$-Lipschitz and $\lambda$ strongly-convex, $A$ satisfies uniform stability $ \gamma\leq \frac{L^2 C^2}{2\lambda n}$ with respect to $\ell$.
\end{theorem}

This stability bound with inverse proportionality to data size can be used estimate $\svk$ with small sample size by observing that $\gamma$-uniform stability directly bounds $|v_i(S)|$ when the utility is given by the test loss $v(S)=-\bigL(A(S))$.
\begin{theorem}\label{thm:mkstronglyconvex}
    Suppose  we have utility function $v(S)=-\bigL(A(S))=-\frac{1}{\abs{D_e}}\sum_{z \in D_e} \ell(w, z)$ where $\ell$ is $\lambda$ strongly-convex and $L$-Lipschitz and $C$ is defined as in Theorem~\ref{thm:sc_stab}. Then,
    \begin{equation}
        m_k \geq \frac{L^4 C^4}{8\lambda^2 a^2k^2}\ln{\frac{2n}{b}}
    \end{equation}
    in Algorithm~\ref{alg:cap} suffices for $(a, b/n)$-approximation of $\svk$ and $(a, b)$-approximation of $\sv_i$.
\end{theorem}

The key point of the theorem is that the impact of adding or removing any one point is small due to stability, and as a result, for a layer of large coalitions, a small sample suffices for high confidence estimates. The final result of approximating the Shapley value in Theorem~\ref{thm:mkstronglyconvex} follows from a union bound applied to computing the average over all layers.

The specific result of $m_k\approx 1/k^2$ implies that in the approximation process, a focus on layers with small coalitions will ensure a good approximation. In the following section, we show how this fundamental idea based on stability extends to more general scenarios including the use of Stochastic Gradient Descent.

\section{Convex and Non-convex optimization with Stochastic Gradient Descent}

As seen in Definition~\ref{def:SV}, Shapley value depends on a well defined utility function $v$ that determines the utility of any subset including the full set $D$ of players. In data valuation, the utility function is the inverse loss of the model computed via an optimization algorithm. When the optimization is randomized -- such as using SGD -- the definition of Shapley value needs some reconsideration.

Suppose a model $M=\A(D)$ computed via the randomized algorithm $\A$ has been released and is widely used. What is the Shapley value of a data point $i$, in the sense of its contribution to the specific model $M$? The randomization in SGD implies inconsistency of Shapley values and marginal utilities. For example, on different runs of $SGD$, the output model $M$ will be different, marginal utilities will be different, and consequently, Shapley values will be different. In these circumstances, there are two possible interpretations of the contribution of a data point.

\subsubsection{Fixed subsequence interpretation} On any $S\subseteq D$ the algorithm $\A$ makes a sequence of random selections $\pi(S)\in \Pi(S)$, where $\Pi(S)$ is the set of all possible random sequences of elements from $S$. For a fixed $\pi$, the consequent algorithm $\A_{\pi(S)}$ operating on $S$ is deterministic, in the sense that $\A_{\pi(S)} (S)$ always produces the same output, and thus $v(S) = -\bigL(\A_{\pi(S)})$.
In the case of SGD, this interpretation corresponds to always selecting the same data points of coalition $S$ in the same sequence. A special case of this fixed sequence interpretation is the case where $\pi(S)$ is a subsequence of the global sequence $\pi(D)$ that produces the model: $M=\A_{\pi(D)}(D)$. For coalition $S$, $\pi(S)$ is the subsequence of $\pi(D)$ consisting only of elements of $S$.

\subsubsection{Expected utility interpretation} The other natural interpretation is in terms of expected utilities. That is, \mbox{$v(S)=-\mathbb{E}_{\pi(S)\in \Pi(S)}[\bigL(\A_{\pi(S)})]$}. In this case, the Shapley value $\sv_i$ not longer measures the contributions to the specific model $M$, but rather the expected contribution of $i$ to the models in the range of $\A$ based on the distribution implied by $\A$. We will see that approximating the expected value is also expensive, requiring many trials of SGD.

\subsection{Sample complexity for Convex SGD}

One of the challenges of estimating $\svk$ for randomized algorithms is that the absolute Uniform Stability bound from Theorem~\ref{thm:sc_stab} no longer applies. In this case, the relevant definition is Uniform Stability in Expectation, where the expectation is taken over the random  sequences in the algorithm $\A$, that is $\Pi(S)$:

\begin{definition}[Definition 2.1 in \cite{hardt_train_2016}]
A randomized algorithm $\A$ is uniformly stable in expectation if for all data sets $S, S^\prime \in Z^n$ such that S and S' are neighboring datasets, there exists an $\epsilon$ such that
\begin{equation}\label{stabdef}
    \sup_z \mathbb{E}_{\Pi(S)}\left[\ell(\A(S);z)-\ell(\A(S^\prime);z)\right] \leq \epsilon,
\end{equation}
    where $\ell(\cdot, z)$ is the loss function and $Z$ is some space.
\end{definition}

SGD can be shown to satisfy such a stability:
\begin{theorem}[Theorem 3.7 in \cite{hardt_train_2016}]\label{thm:stabconvex}
    Assume that the loss function $\ell(\cdot; z)$ is $\beta-$smooth, convex and $L-$Lipschitz for every $z$. Suppose that we run SGD with step sizes $\alpha_t \leq 2/\beta$ for $T$ steps. Then, SGD satisfies expected uniform stability with
    $
        \epsilon \leq \frac{2L^2}{k}\sum_{t=1}^T\alpha_t,
    $
    where $w=\A(S), w^\prime=\A(S^\prime)$ and $k$ is the coalition size.
\end{theorem}
The results in this section assume a  loss bounded within limits $\ell(w, z)\in [0, G]$, which is a common assumption in theoretical analysis including data valuation~\cite{hardt_train_2016,kwon2022beta,wang2023data}.

\subsubsection{Sample complexity for fixed subsequences}
We consider the setup where $\pi$ is a function that returns for every $S$ a fixed sequence $\pi(S)$, and $\pi(S)$ is a uniformly random choice from $\Pi(S)$. The stability in expectation concept can be applied to yield the following theorem:

\begin{theorem}\label{size_convex}
     When using Algorithm~\ref{alg:cap} to estimate $\svk$ for an SGD algorithm $\A_\pi$ with a convex, $\beta-$smooth and $L-$Lipschitz loss function, then
    \begin{equation*}
        m_k \geq \brackets{\frac{\frac{32T^2L^4}{\beta^2k^2}+\frac{8\s TL^2}{k\beta}+4\s a/3}{a^2}}\ln{\frac{2n}{b}},
    \end{equation*}
    suffices for $\Pr(\abs{\svkh - \svk}\geq a)\leq b/n$, and by union bound, it suffices to estimate the Shapley value:
    $ Pr\brackets{\abs{\hat{\sv}_i-\varphi_i} \geq a}\leq b.
    $
\end{theorem}

Since the absolute limit on stability does not hold (which is required for the Chernoff-Hoeffding bound), the proof of the theorem uses the Bernstein bound. The theorem assumes that the query item $i$ is inserted at an arbitrary position in the sequence $\pi(S)$. Next, building upon this result, we will address the expected utility scenario.

\subsubsection{Sample complexity for expected utility}
Suppose $S$ is a coalition of size $k$ not containing $i$, and $S^\prime=S\cup i$. Then the expected marginal utility can be computed as:

\begin{equation}
    v_i(S) = \frac{1}{(k+1)!} \sum_{\pi(S^\prime) \in \Pi(S^\prime)} \bigL(\A_{\pi(S^\prime)}(S)) - \bigL(A_{\pi(S^\prime)}(S^\prime))
\end{equation}
where $\A_{\pi(S^\prime)}(S)$ is computed by inserting a suitable {\em null} element at the position of $i$. In the SGD context, the zero gradient (implemented simply by excluding $i$ in the sequence) is the natural choice. This marginal utility is expensive to compute. We can estimate it by taking a sample $H\subseteq \Pi(S^\prime)$ of size $\abs{H}=\vsample$:
$  \hat{v}_i(S) = \frac{1}{\vsample}\sum_{\pi(S^\prime) \in H} \bigL(\A_{\pi(S^\prime)}(S)) - \bigL(A_{\pi(S^\prime)}(S^\prime))$

The following theorem provides an estimate of the number of permutation samples $\vsample$ needed based on the stability of SGD:
\begin{theorem}\label{thm:t_0variance_varphi}
Suppose that, under the fixed subsequence interpretation, the number of samples required to compute $\svkh$ such that $Pr\brackets{\abs{\svkh-\svk} \geq \frac{a}{2}} \leq \frac{b}{2n}$ using an algorithm satisfying expected uniform stability with constant $\epsilon$ and loss function $\ell \in [0,\s]$ is $m_k$, then we can sample
    \begin{equation}\label{eq:singlecoalition}
        \vsample \geq \frac{8\brackets{\bound+\s a/3}}{a^2}\ln{\frac{4n m_k}{b}}
    \end{equation}
    permutations of each coalition $S \in \hat{S}_k$, where $\hat{S}_k$ is the set of $m_k$ coalitions sampled to compute $\svkh$ such that we can compute estimate $\svkho$ of $\svkh$ satisfying $Pr\brackets{\abs{\svkho-\svk} \geq a} \leq \frac{b}{n}$.
\end{theorem}
\begin{corollary}\label{cor:repeated_sampling_convex}
    For each layer the number of evaluated samples necessary to estimate $\svkho$ such that $Pr\brackets{\abs{\svkho-\svk} \geq a} \leq \frac{b}{n}$ is at least $m_k\frac{8\brackets{\bound+\s a/3}}{a^2}\ln{\frac{4n m_k}{b}}$. Therefore, for convex optimisation using SGD, the number $h$ of evaluated permutations per layer $k$ must satisfy
    \begin{equation}
        h \geq m_k\frac{8\brackets{\frac{16T^2L^4}{\beta^2k^2}+\frac{4\s TL^2}{k\beta} + \s a /3}}{a^2}\ln{\frac{4nm_k}{b}}
    \end{equation}
    under the assumptions of Theorems \ref{thm:stabconvex} and \ref{thm:t_0variance_varphi}.
\end{corollary}

Thus, in the expected utility case, the randomness of $SGD$ itself introduces uncertainty and multiples the model trainings required to estimate $\svk$ and therefore $\sv_i$.

  \subsection{Sample complexity for non-convex SGD}

SGD with non-convex loss function is one of the most important algorithmic problems in machine learning as many practical models classes -- including neural networks -- have a non-convex loss function. A non-convex function can have multiple minima, and it is hard to guarantee the quality of results from a gradient descent process. For such algorithms the following bound is known for expected uniform stability.

\begin{theorem}[Theorem 3.8 in \cite{hardt_train_2016}]\label{Hardtbound}
Assume that $\ell(\cdot; z) \in [0, \s]$ is an L-Lipschitz and $\beta$-smooth loss function for
every z. Suppose that we run SGD for T steps with monotonically non-increasing step sizes $\alpha_t \leq c/t$. Then, SGD satisfies uniform stability in expectation with
    \begin{equation}
         \epsilon \leq \s^{\frac{\beta c}{\beta c+1}}(2cL^2)^{\frac{1}{\beta c+1}}T^{\frac{\beta c}{\beta c+1}}\frac{1+\frac{1}{\beta c}}{k-1} = H_k,
    \end{equation}
    where $k$ is the size of the dataset.
\end{theorem}
Note that due to $H_k$ being undefined for $k=1$, we restrict our analysis to layer sizes $2\dots,n-1$ in this section.

\subsubsection{Sample complexity for fixed subsequences with non-convex loss}
As in the convex case, we can employ Bernstein's inequality to derive a sample size.

\begin{theorem}\label{m_k size}
    Assume an algorithm fulfilling the assumptions of Theorem \ref{Hardtbound} is employed to compute Shapley values using Algorithm \ref{algo:stratified-sampling}. In each layer $k$, let $m_k$ be the number of coalitions sampled and let $\svkh$ be the estimate of $\svk$. Then, if
    \begin{equation}
        m_k \geq 2\ln{\frac{2n}{b}}\brackets{\frac{2\hat{H}_k^2+2\s \hat{H}_k+4\s a/3}{a^2}},
    \end{equation}
    where
    \begin{equation}
        \hat{H}_k =\s^{\frac{\beta c}{\beta c+1}}(2cL^2)^{\frac{1}{\beta c+1}}T^{\frac{\beta c}{\beta c+1}}\frac{1+\frac{1}{\beta c}}{k-1},
    \end{equation}
    then
    \begin{equation}
        Pr\brackets{\abs{\hat{\sv}_i-\sv_i} \geq a}\leq b
    \end{equation}
\end{theorem}

\subsubsection{Sample complexity for expected utility with non-convex loss}

In expected utility computation, the randomness of SGD again increases the computational complexity. The results for non-convex case are analogous to the convex case, but adjusted for the stability bounds for non-convex loss.
\begin{corollary}
    The number $h$ of permutations per layer $k$ must satisfy
    $ \vsample \geq \frac{8\brackets{\hat{H}_k^2+\s\hat{H}_k+\s a/3}}{a^2}\ln{\frac{4n m_k}{b}}$
    for non-convex optimization under the assumptions of Theorems \ref{thm:t_0variance_varphi} and \ref{m_k size}.
\end{corollary}

\section{Restricted Layer Selection and $\delta-$Shapley}

In the previous section we saw that in multiple different kinds of learning scenarios, the expected stability of models trained on a dataset of size $k$ is at most $O(1/k)$. As a result, the expected marginal utility of any data point is also bounded by $O(1/k)$ in expectation. This observation suggests that on average it is the smaller coalition sizes that may be the main contributors to the Shapley value of data points.

On the other hand, when coalitions are very small, the stability is poorer and the noise from randomization of SGD is likely to have a greater effect. Further, it is generally known that in machine learning, small training datasets often produce noisy models, in the sense that the model may overfit to the training data, but generalise poorly to test data that is used to evaluate utility. This known fact suggests that marginal gains computed on small coalitions will be noisy, see the Appendix for experiments and further discussion.

These observations suggest that models in a certain {\em medium} size range are most likely to be effective -- where the coalitions are large enough to generalise well, but not so large that the marginal utility disappears. Thus, we can define a define a different value to approximate $\sv_i$ that only evaluates a selection of layers:

\begin{definition}[$\delta$-Shapley value]\label{def:deltashap}
    The $\delta$-Shapley value of datapoint $i$ in a training set $D$ of size $n$ with utility function $v$ is given by  \begin{equation}
        \varphi_i=\sum_{k=1}^n p_k \sum_{\substack{S\subseteq D\setminus\{i\} \\ |S|=k-1}} v_i(S)
    \end{equation}
    with
    \begin{equation}
        p_k=\begin{cases}
    0, & \text{if $k<B_L$ or $k>B_U$}.\\
    \frac{1}{B_U-B_L+1}, & \text{otherwise}.
        \end{cases}
    \end{equation}

    where $B_u$ represents the upper bound on layersize $k$ being evaluated and $B_L$ is the corresponding lower bound.
\end{definition}

Definition~\ref{def:deltashap} differs from the Shapley value as not all coalition sizes are evaluated, therefore belonging to a more general class of indexes that also includes the Shapley value; semi-values~\cite{dubey1981value}.
\begin{definition}~\cite{dubey1981value}
    A value function $\varphi$ of a set of size $n$ is a semivalue if and only if there exists a weight function $w:[1:n]\rightarrow\mathbb{R}$ such that $\sum_{k=1}^n {n-1 \choose k-1}w(k)=n$ and $\varphi_i$ can be expressed as follows:
    \begin{equation}
        \varphi_i=\sum_{k=1}^n \frac{w(k)}{n} \sum_{\substack{S\subseteq D\setminus\{i\} \\ |S|=k-1}} v_i(S)
    \end{equation}
\end{definition}
The Shapley value is the semi-value given by $w(k)={n-1 \choose k-1}^{-1}$. In general, semi-values must satisfy only three of the four Shapley axioms, excluding efficiency.\footnote{See~\citet{kwon2022beta,wang2023data} for a discussion of the relevance of the efficiency property for data valuation.} Semi-values can be viewed as a reweighting of the importance of each layer (each coalition size) in the Shapley value calculation, multiplying each marginal contribution by a value $w_k$ for layer $k\in[1:n-1]$.

Our previous results describe how many coalitions of each size $k$ should be evaluated to ensure an accurate estimation of the marginal contribution $v_i(S)$ of $i$ for layer $k$. These results can therefore be used to estimate how many evaluations will be needed to ensure an accurate approximation of the $\delta$-Shapley for each layer $k$.

\begin{figure*}[ht!]
\centering
\begin{tabular}{cccc}
\includegraphics[width=1.5in]{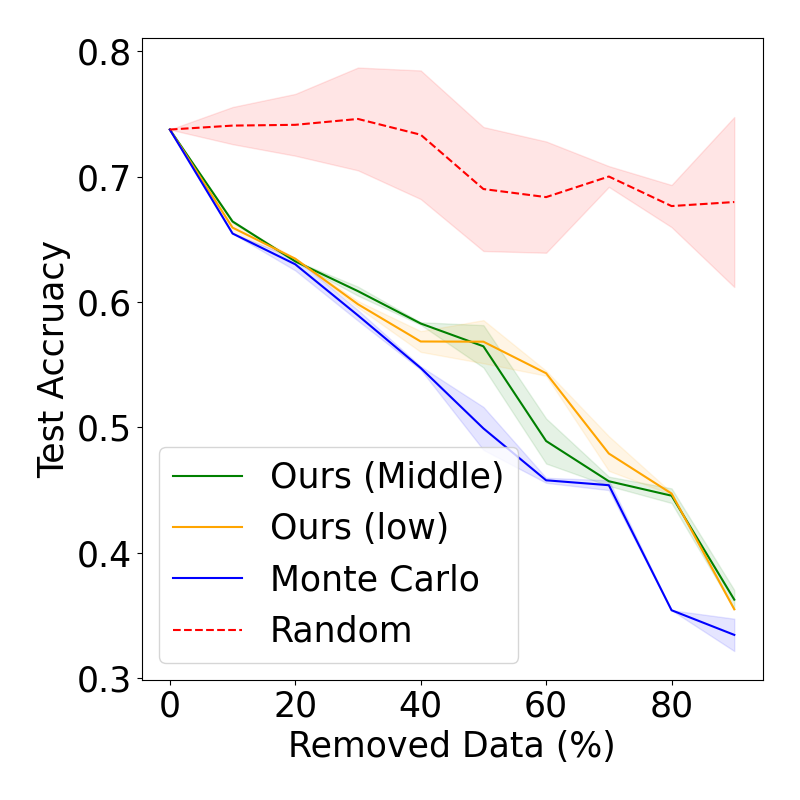} &
\includegraphics[width=1.5in]{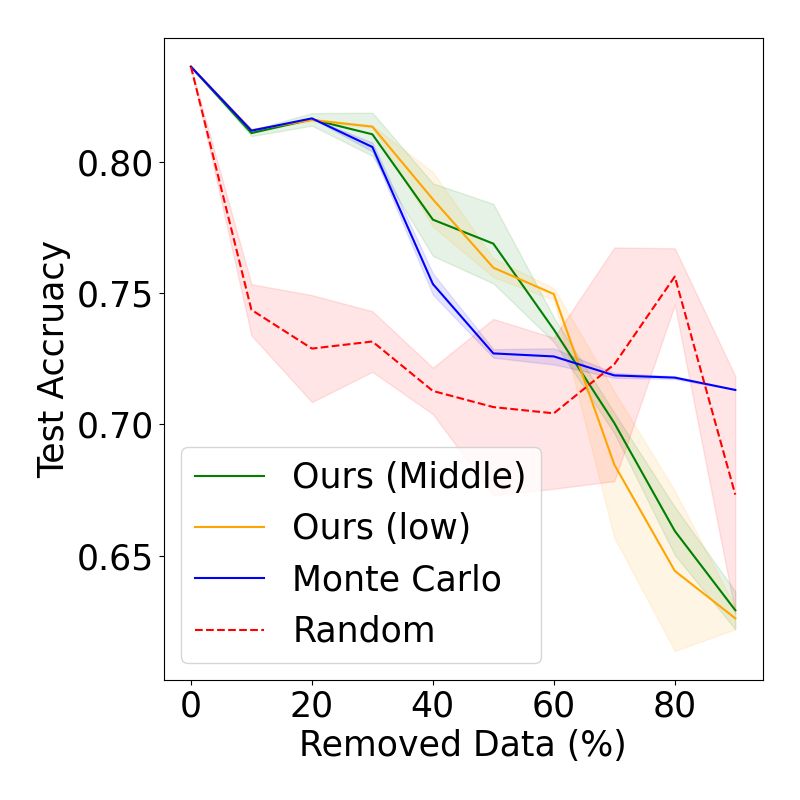}&\includegraphics[width=1.5in]{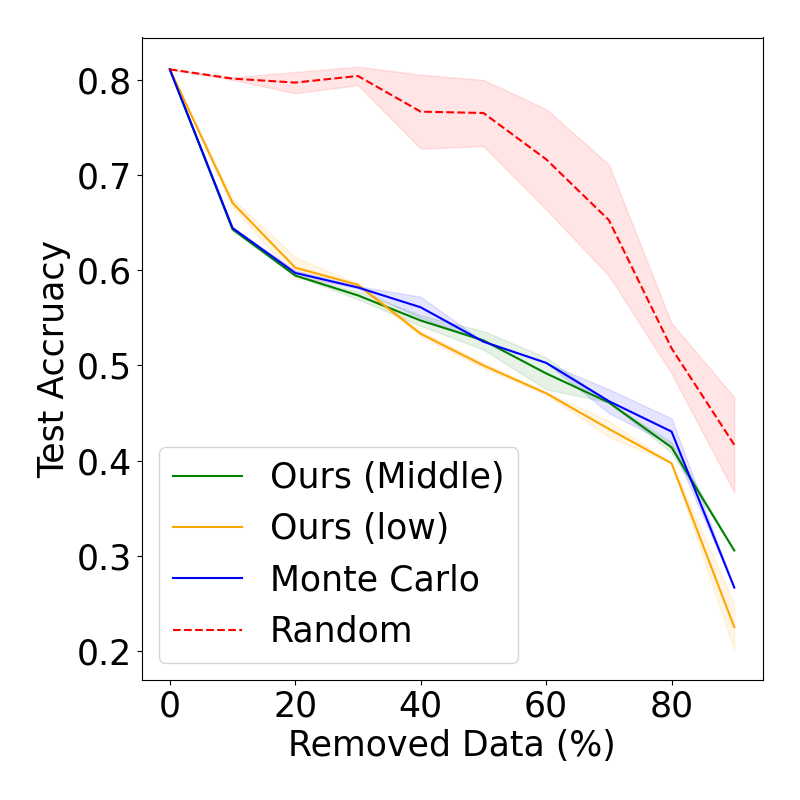}
&\includegraphics[width=1.5in]{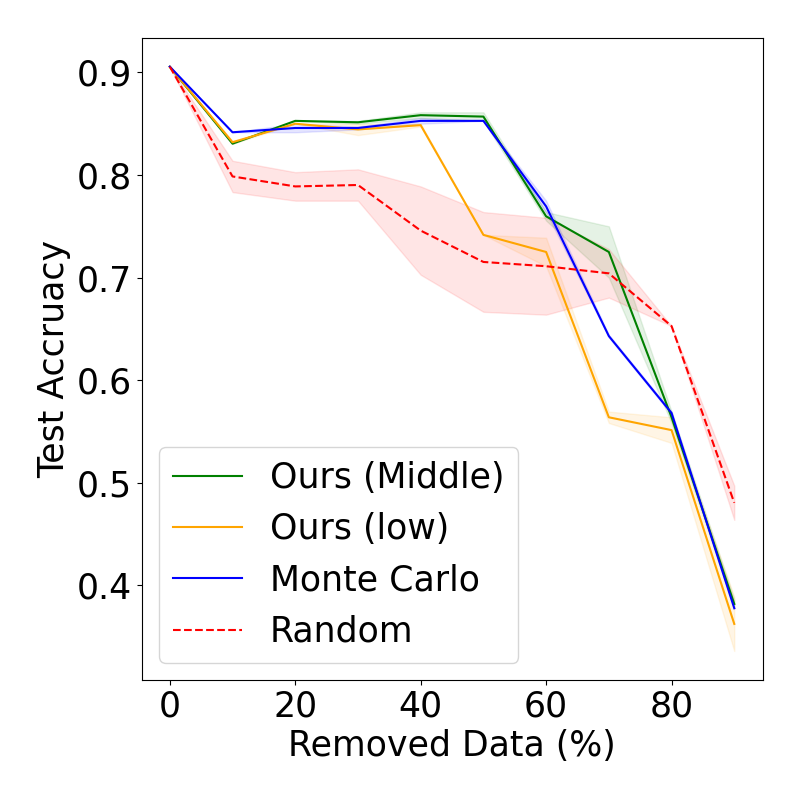}\\
(a) Adult (Large first) &  (b) Adult (Small first) & (c) Digits (Large first) & (d) Digits (Small first)  \\
\includegraphics[width=1.5in]{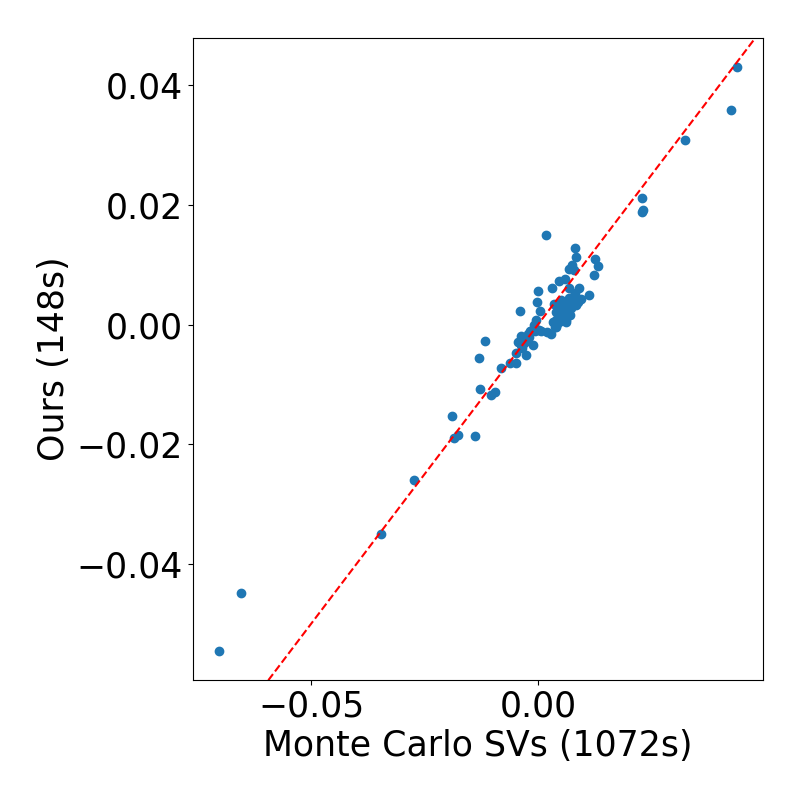} &\includegraphics[width=1.5in]{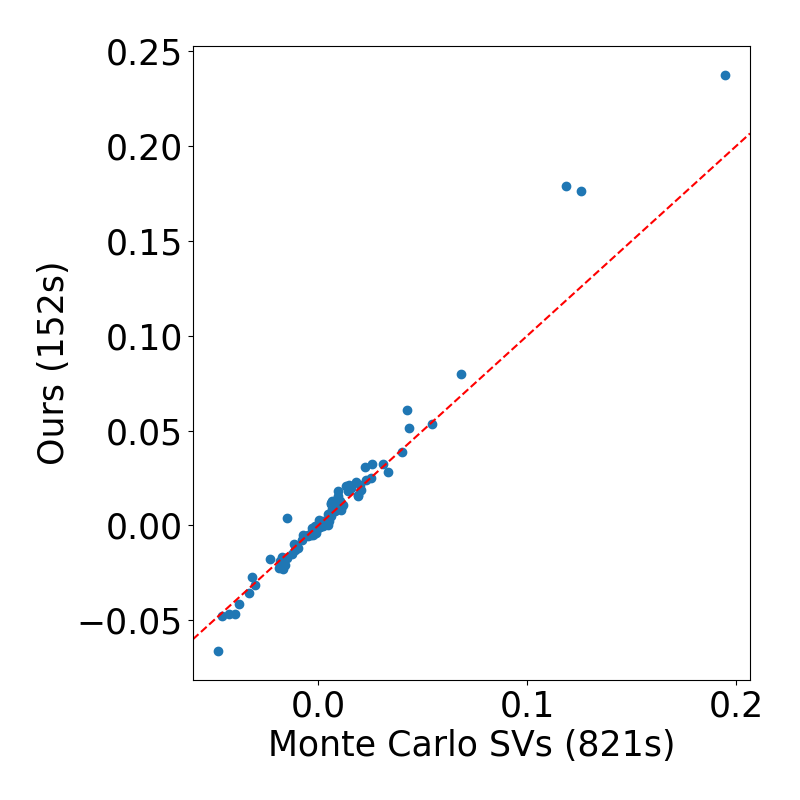} &
\includegraphics[width=1.5in]{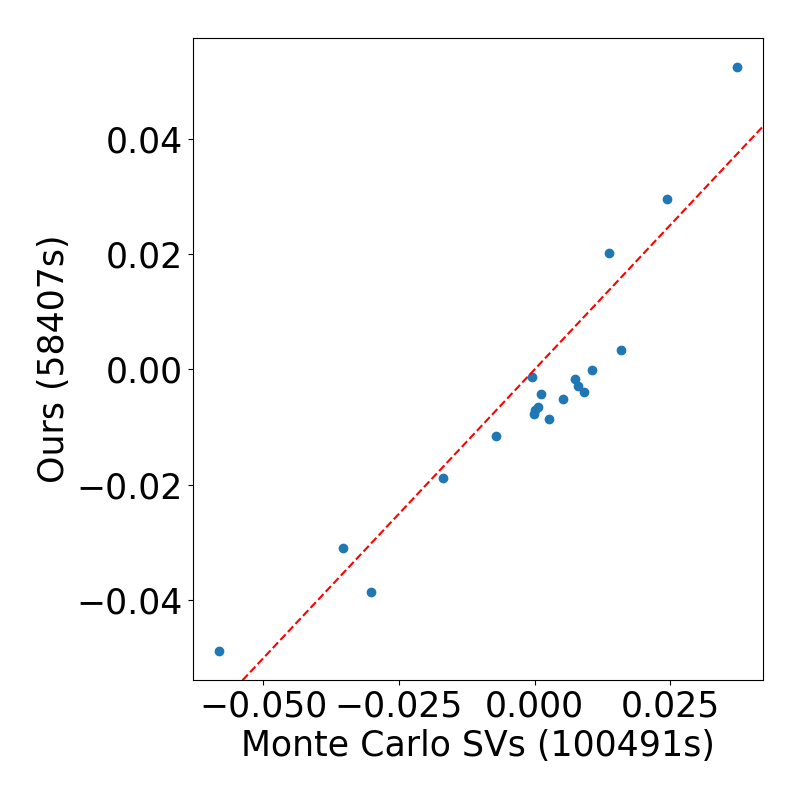} &
\includegraphics[width=1.5in]{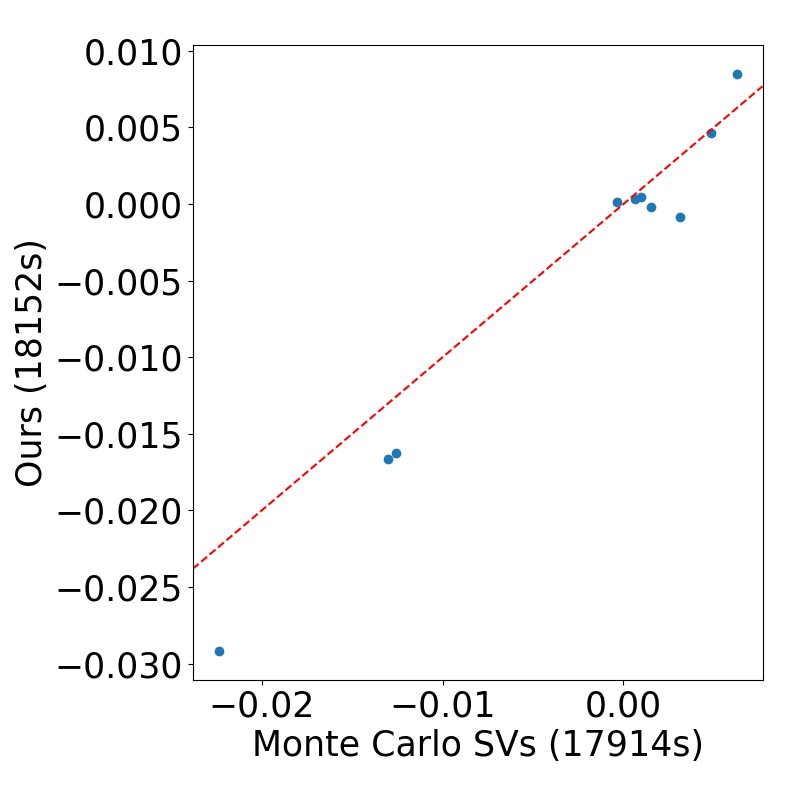}\\
e) Adult ($\rho=0.88$)  &f) Digits ($\rho=0.98$) & (g) BFM ($\rho=0.91$) & (h) BFM (pretrain, $\rho=0.84$)\\
\end{tabular}
 \caption{Change in accuracy due to removing points ranked by their Shapley value. The top row shows the effect on test accuracy when points with the highest or lowest values are removed first. We expect meaningful Shapley values to result in lines below and above random values respectively. The bottom row contains example correlation plots, with $y=x$ in red. }
\label{fig:sc-accs}
\end{figure*}

\section{Experiments}\label{Experiments}
The experiments outlined in this section empirically verify that Shapley values can be approximated via $\delta$-Shapley using relatively small coalitions. Specifically:
\begin{enumerate}
    \item Meaningful data valuations: they find relevant valuations of the dataset in the sense that they reliably identify points causing the greatest drop in test accuracy when removed from the training set.
    \item Accurate approximations of Monte Carlo Shapley values: they have high Spearman Rank correlation with the corresponding Monte Carlo Shapley values.
    \item Efficient: they are obtained using significantly less computation time (wall-clock time) in comparison to the Monte Carlo approach.
    \item On pretrained neural networks, even smaller coalitions suffice and computation is more efficient.
\end{enumerate}
\subsection{Experimental Setup}
\noindent\textbf{Datasets and Models: } All experiments use publicly available  classification datasets. In the convex case, we report results on the binary classification task given by the  Adult dataset ~\cite{kohavi1996scaling} as well as the Breast Cancer Wisonsin binary classification task
 and the 10 class Digits dataset which is a lightweight version of MNIST provided by the UCI Machine Learning Repository~\cite{frank2010uci}. In the non-convex case, we use a binary classification version of the FashionMNIST ~\cite{xiao2017} dataset, as described by~\citet{ghorbani_data_2019}, comparing class 0 (tops/t-shirts) to class 6 (shirts). Convex experiments implement regularized logistic regression with $\lambda=0.1$ via Scikit-Learn~\cite{pedregosa2011scikit} . For binary-fashionMNIST, a CNN with max-pooling and RELU activations is trained (PyTorch~\cite{paszke2017automatic} ). Further details on both the model and training (e.g hyperparameters) can be found in both the Appendix and codebase. Experiments were averaged over 3-5 runs.

\noindent\textbf{Sampling Strategies: } We approximate the $\delta$-Shapley value by sampling a layersize $k$ uniformly between the upper ($B_U$) and lower ($B_L$) layersize bounds in each iteration with $B_L=\frac{n}{3}$ and $B_U=\frac{2n}{3}$ (\textit{middle} sampling) and also with $B_L=\frac{2n}{10}$ and $B_U=\frac{3n}{10}$ (\textit{low} sampling).  The $\delta$-Shapley approximation is the average marginal contribution over  iterations.

\noindent\textbf{Comparison Algorithms: } The baseline comparison algorithm assigns a random shapley value $\phi_i\sim$Uniform($[0,1)$) to each datapoint. We also compare to the popular Monte Carlo sampling approach~\cite{ghorbani_data_2019} discussed above. The Monte Carlo algorithm samples permutations of the full dataset, selecting a comparison coalition for each datapoint being evaluated based on its position in this permutation.\footnote{When evaluating only a subset of datapoints from the entire dataset (e.g. for non-convex models) then only the subsets within each permutation corresponding to those points are used. This is in contrast to cycling through the entire permutation, to ensure that time-complexities are comparable.} The algorithm converges when the average deviation has converged, $\frac{1}{n}\sum_{i=1}^n \frac{\phi_i^t - \phi_i^{t-100}}{|\phi_i^t|} <0.05$.  This convergence criteria is used for all algorithms to ensure fair comparison, particularly in terms of wall-clock time.

We apply the \textit{fixed subsequence interpretation} of data valuation discussed previously. In our experiments we fix the order of the coalition selected in each marginal contribution calculation, varying only the place that the new datapoint is inserted into the coalition between $v(S)$ and $V(S\cup i)$.

\begin{table}[h!]
\centering
\begin{tabular}{llll}
\cline{1-4}
\textbf{Dataset    }    & \textbf{MC}(s) & \textbf{Ours}(s) & $\rho$\\ \cline{1-4}
Adult(mid)       &    1021($\pm$38)      &  153($\pm$5) &  0.87 ($\pm$.01)\\

Adult(low)       &   1021($\pm$38)     & 215($\pm$43)   & 0.85($\pm$.01)\\
Digits(mid)       &   817($\pm$5)   &   139($\pm$13) &0.98($\pm$.004)  \\
Digits(low)       &   817 ($\pm$5)    &   82($\pm$9) &0.90($\pm$.005)\\
Cancer(mid)       &   586($\pm$13)   & 263($\pm$43)     & 0.85($\pm$0.02) \\
Cancer(low)       &  586($\pm$13)     &108($\pm$43)   & 0.84($\pm$0.01)\\
\cline{1-4}

BFM(mid)& 100632(151)& 60474(206)&0.91 ($\pm.01$)\\
BFM(pre)&     71916(814) & 18152(245)&0.84 ($\pm.02$)
\end{tabular}
\caption{Wall-Clock Time (s) and Spearman correlations ($\rho$). }\label{tab:results}
\end{table}

\subsection{Results}
\noindent\textbf{Convex:} In the convex case, Shapley values obtained by the $\delta$-Shapley algorithm are shown in Figure~\ref{fig:sc-accs} to be comparable with those obtained using the Monte Carlo approach. Both sets of Shapley values reliably identify high and low value points in comparison to random selection. Table~\ref{tab:results} demonstrates that the Monte Carlo and $\delta$-Shapley values are highly correlated with average correlations of between $0.87$ and $0.98$ for middle sampling and $0.85$ and $0.90$ for low sampling. This implies that similar information is obtained by both algorithms. Table~\ref{tab:results} shows up to a 9.9x reduction in computation time using the $\delta$-Shapley algorithm, with significant computation time reductions across all datasets.

\noindent\textbf{Non-Convex:} In the non-convex case for binary-FashionMNIST,  Figure~\ref{fig:sc-accs} g) and Table~\ref{tab:results} show a high average level of correlation of $0.91$ with a significant reduction in computation time.

In this case, we also examined the effect of using pre-trained models on $\delta-$Shapley as pre-training is an increasingly common paradigm in deep learning. The results in Figure~\ref{fig:sc-accs} h) and Table~\ref{tab:results} show that $\delta$-Shapley values obtained for pre-trained models using the low sampling strategy (coalitions size $\approx 20$) correlate highly with Monte Carlo Shapley values ($\rho=0.84$) with significantly reduced computation time (4x).

\section{Related Work}
The Shapley value was originally proposed in game theory to estimate the contribution of players in cooperative games~\cite{shapley1953value}.  In this general setting, Monte Carlo estimation is a popular Shapley value approximation method~\cite{castro2009polynomial,maleki_bounding_2014} and several variance reduction techniques have been proposed that work by grouping together (or otherwise linking) similar permutations, coalitions or players~\cite{burgess2021approximating, CASTRO2017180, illes2019estimation}. Recently, the Shapley value has been applied to attribution problems in Machine Learning including explainability~\cite{lundberg2017unified},  feature selection~\cite{cohen2007feature}, federated learning~\cite{wang2020principled}, ensemble pruning~\cite{rozemberczki2021shapley} and multi-agent reinforcement learning~\cite{shapcredit}.

\citet{ghorbani_data_2019} introduced Shapley values for data valuation, which has seen many applications discussed earlier. Alternative data valuation methods include leave-one-out testing~\cite{cook1977detection} and influence function estimation~\cite{sharchilev2018finding}. Recently other valuation functions have been proposed such as the Beta-Shapley~\cite{kwon2022beta} and Banzhaf value~\cite{wang2023data} which are semi-values with different weighting functions for the layers. These valuations are designed to improve ranking for use in machine learning and data selection. In comparison, our aim was to approximate the Shapley value itself which is important in applications such as valuation for fair compensation.

\section{Conclusion}

The important insight from this work is that small coalitions are good estimators of the data value due to the diminishing returns property of data value with set size. Similar properties and results may be relevant in other applications of Shapley value. An area that we believe requires further investigation is the randomness properties of SGD and consequent variations in value of data points. This question is also likely to have impact in other areas such as privacy.

\bibliography{refs.bib}

\newpage\phantom{blabla}
\newpage\phantom{blabla}
\section{Appendix}

\maketitle

\subsection{Proof of Theorem \ref{thm:mkstronglyconvex}}
\begin{remark}
        Suppose  we have utility function $v(S)=-\bigL(A(S))=-\frac{1}{\abs{D_e}}\sum_{z \in D_e} \ell(w, z)$ where $\ell$ is $\lambda$ strongly-convex and $L$-Lipschitz and $C$ is defined as in Theorem~\ref{thm:sc_stab}. Then,
    \begin{equation}
        m_k \geq \frac{L^4 C^4}{8\lambda^2 a^2k^2}\ln{\frac{2n}{b}}
    \end{equation}
    in Algorithm~\ref{alg:cap} suffices for $(a, b/n)$-approximation of $\svk$ and $(a, b)$-approximation of $\sv_i$.
\end{remark}
    \begin{proof}
        It follows from the assumption that $\abs{v_i(S)} \leq \gamma$ and Hoeffding's Inequality \cite{hoeffding_probability_1963} that because $\E{\svkh} = \svk$
        \begin{equation}
            Pr\brackets{\abs{\svkh-\svk}\geq a} \leq 2\exp\brackets{-\frac{2m_k a^2}{\gamma^2}}
        \end{equation}
         To  see why, to obtain a $(a,b)-$bound, we need to need to find $m_k$ such that
        \begin{equation}
            2\exp\brackets{-\frac{2m_k a^2}{\gamma ^2}} \leq \frac{b}{n},
        \end{equation}
        consider the following Lemma.
        \begin{lemma}\label{lem:unionbound}
In each layer $k$ of the $n$ layers, let $\svk$ be Shapley value of datum $i$ in this layer. Then, if we obtain an estimate $\svkh$ such that
    \begin{equation}\label{eq:layerkbound}
    Pr\brackets{\abs{\svkh-\svk} \geq a} \leq \frac{b}{n}
\end{equation}
in each layer, then
\begin{equation}
    Pr\brackets{\abs{\hat{\varphi}_i-\varphi_i} \geq a} \leq b.
\end{equation}
\begin{proof}
By the union bound, the probability that for at least one layer $k$, $\abs{\svkh-\svk} \geq a$ holds is
\begin{equation}
    Pr\brackets{\exists k \in \{0,\dots,n-1\}: \abs{\svkh-\svk} \geq a} \leq \sum_{k = 0}^{n-1} \frac{b}{n} = b
\end{equation}
    Note that we can write $\varphi_i = \frac{1}{n}\sum_{k = 0}^{n-1} \svk$, thus $\hat{\varphi_i}-\varphi_i = \frac{1}{n}\sum_{k = 0}^{n-1} \brackets{\svkh - \svk}$ and
    \begin{equation}
        \abs{\hat{\varphi_i}-\varphi_i} = \abs{\frac{1}{n}\sum_{k = 0}^{n-1} \brackets{\svkh - \svk}} \leq \frac{1}{n}\sum_{k = 0}^{n-1} \abs{\svkh - \svk}.
    \end{equation}
    Further, we observe that the event $\sum_{k=0}^{n-1} \abs{\svkh-\svk} \geq na$ requires that there exists at least one $k$ such that $\abs{\svkh-\svk} \geq a$. We can therefore conclude that
    \begin{subequations}
    \begin{align}
        &Pr\brackets{\abs{\hat{\varphi}_i-\varphi_i} \geq a} \leq Pr\brackets{\frac{1}{n}\sum_{k = 0}^{n-1} \abs{\svkh - \svk} \geq a} =\\ &Pr\brackets{\sum_{k = 0}^{n-1} \abs{\svkh - \svk} \geq na} \leq b
        \end{align}
    \end{subequations}
    Note that we are omitting the layer $k=0$ in the stability analysis, as it does not make sense in the context of data valuation. We therefore assume $\phi_i^0 = 0$ for all $i$.
\end{proof}
\end{lemma}
        We can thus derive $m_k$
        \begin{equation}
            \frac{2m_k a^2}{\gamma^2} \geq \ln{\frac{2n}{b}} \Rightarrow m_k \geq \frac{\gamma^2}{2a^2}\ln{\frac{2n}{b}}
        \end{equation}
    Substituting $\gamma^2 = \frac{L^4C^4}{4\lambda^2k^2}$ yields the stated result.
    \end{proof}

\subsection{Proof of Theorem \ref{size_convex}}
The proof of Theorem \ref{size_convex} uses an upper bound on the expectation of $v_i(S)^2$ together with a concentration inequality to derive the desired convergence guarantee.
We first need to state some preliminary results.
First we derive a bound on $\E{(\ldiff)^2}$ for algorithms satisfying expected uniform stability in the following Lemma.

\begin{lemma}\label{quaddiffbound}

    If $\A$ is an algorithm satisfying stability in expectation with constant $\epsilon$ and $\ell(w_T = \A(S); z)$ is a loss function such that $\ell \in [0,\s]$ then
    \begin{equation}
    \E{(\ldiff)^2} \leq \epsilon^2 + \epsilon \s.
    \end{equation}
        \begin{proof}
     It follows from the definition of variance that
    \begin{equation}\label{eq:varxbound}
       \E{X^2} = \mathrm{Var}(X)+(\E{X})^2.
    \end{equation}
    We first notice that $\E{\abs{\ldiff}^2} = \E{(\ldiff)^2}$ and will thus use $X = \E{\abs{\ldiff}}$ to prove the statement.
    We show how to bound both summands above. First, the Bathia-Davis inequality \cite{bathia}  states that for a random variable $Y$ with expectation $\mu$ bounded above by $M$ and below by $m$, it holds that:
    \begin{equation}
        \mathrm{Var}(Y) \leq (M-\mu)(\mu-m).
    \end{equation}
    The upper bound is maximized when $(M-\mu) = M$ and $(\mu-M)=\mu$. As $\abs{\ldiff}$ is a random variable bounded by $[0,\s]$ with expectation bounded by $[0, \epsilon]$, it follows that the upper bound is maximized when $(M-\mu)(\mu-m) = \epsilon\s$. Thus, we know that $\mathrm{Var}(\abs{\ldiff}) \leq \epsilon\s $. It also follows from the fact that $\E{\abs{\ldiff}} \leq \epsilon$ that $(\E{\abs{\ldiff}})^2 \leq \epsilon^2$. Substituting the bounds for the summands in Equation \ref{eq:varxbound} yields the stated result.
    \end{proof}
\end{lemma}

As we are using the validation loss $\bigL$ as our valuation metric, we extend this Lemma to hold for averages of losses in the next Lemma.
\begin{lemma}\label{lem:LVar}
    Let $\A$ be a $\epsilon-$uniformly stable algorithm. If we use $v_i(S) = \bigL - \bigL'$ to measure the marginal contribution, then
    \begin{equation}
        \E{\abs{\bigL-\bigL'}} \leq \epsilon
    \end{equation}
    and
    \begin{equation}
        \E{(\bigL-\bigL')^2} \leq \bound
    \end{equation}
    \begin{proof}
        We first rewrite $\E{\abs{v_i(S)}}$ as:
        \begin{subequations}
        \begin{gather}
            \E{\abs{\bigL-\bigL'}}
            = \E{\abs{\frac{1}{\abs{D_e}}\sum_{z \in D_e} \ell(w;z)-\ell(w';z)}} \\
            = \frac{1}{\abs{D_e}}\E{\abs{\sum_{z \in D_e}\ell(w;z)-\ell(w';z)}}\\ \leq \frac{1}{\abs{D_e}}\E{\sum_{z \in D_e}\abs{\ell(w;z)-\ell(w';z)}} \label{ineq:absEV}\\
            = \frac{1}{\abs{D_e}}\sum_{z \in D_e}\E{\abs{\ell(w;z)-\ell(w';z)}}  \\
            = \E{\abs{\ell(w;z)-\ell(w';z)}} \leq \epsilon.
            \end{gather}
        \end{subequations}
        Where line \ref{ineq:absEV} follows from the triangle inequality.
    Following from the assumption that the $\ldiff$ are independent, we thus have:
    \begin{subequations}\label{}
    \begin{gather}
        \E{\brackets{\bigL-\bigL'}^2} \\ = \E{\brackets{\frac{1}{\abs{D_e}}\sum_{z \in D_e} \ell(w;z)-\ell(w';z)}^2} \\ =\frac{1}{\abs{D_e}^2}\E{\brackets{\sum_{z \in D_e} \ell(w;z)-\ell(w';z)}^2}
        \end{gather}
        \end{subequations}
        It follows from Cauchy's Inequality \cite{handbook-mathematical} that
        \begin{subequations}
            \begin{gather}
        \frac{1}{\abs{D_e}^2}\brackets{\sum_{z \in D_e} \ell(w;z)-\ell(w';z)}^2 \\\leq \frac{1}{\abs{D_e}}\sum_{z \in D_e}\brackets{\ell(w;z)-\ell(w';z)}^2.
    \end{gather}
    \end{subequations}
    Due to the linearity of expectation, we proceed with
    \begin{subequations}
            \begin{gather}
        \frac{1}{\abs{D_e}^2}\E{\brackets{\sum_{z \in D_e} \ell(w;z)-\ell(w';z)}^2} \\\leq \frac{1}{\abs{D_e}}\sum_{z \in D_e}\E{\brackets{\ell(w;z)-\ell(w';z)}^2} \leq \bound.
    \end{gather}
    \end{subequations}
    The last step follows from Lemma \ref{quaddiffbound}.
    \end{proof}
\end{lemma}
We now have all necessary components to obtain a one-sided convergence guarantee of $\svkh-\svk$ using Bernstein's Inequality as stated below.

\begin{theorem}[Bernstein, Proposition 2.14 in \cite{wainwright_2019}]\label{bernstein}
Given m independent random variables such that $X_i \leq M$ almost surely,
we have
\begin{subequations}
\begin{gather}
    Pr\left(\sum_{i=1}^m\brackets{X_i - \E{X_i}} \geq ma \right) \\ \leq \exp{\left(\frac{-ma^2/2}{\frac{1}{m}\sum_{i=1}^m\E{X_i^2}+Ma/3}\right)}
\end{gather}
\end{subequations}
\end{theorem}
Since we are interested in the absolute difference between our estimate and the true Shapley Value, we need to obtain a result for $\abs{\sum_{i=1}^m\brackets{X_i - \E{X_i}}}$, which we derive in the next Lemma:

\begin{lemma}\label{Bernstein_applied}
    Let $X_1,\dots,X_m$ be independent and identically distributed real-valued random variables such that $X_i \leq M$ almost surely for all $i \leq m$ and let , then
    \begin{subequations}
    \begin{gather}
        Pr\left(\abs{\sum_{i=1}^m X_i - m\E{X}} \geq ma \right) \\ \leq 2\exp{\left(-\frac{ma^2}{2(\frac{1}{m}\sum_{i=1}^m\E{X_i^2}+Ma/3)}\right)}
    \end{gather}
    \end{subequations}
\end{lemma}
\begin{proof}
    Because the random variables $X_1,\dots,X_m$ are independent and identically distributed, the expected value $\E{X}$ is the same for all $X_i$. We can thus write
    \begin{equation}
    \sum_{i=1}^m\brackets{X_i - \E{X_i}} = \sum_{i=1}^m X_i - m\E{X}
    \end{equation}
    We note that for $A = \sum_{i=1}^m X_i - m\E{X}$, it holds that
    \begin{equation}
        Pr\left(\abs{A} \geq ma \right) = Pr\left(A \leq -ma \right) + Pr\left(A \geq ma \right)
    \end{equation}
    and
    \begin{equation}
     Pr\left(A \leq -ma \right) = Pr\left(-\left(\sum_{i=1}^m X_i - m\E{X}\right)\geq ma \right)
    \end{equation}
    When we label each $Y_i := -X_i$, then the conditions of Theorem \ref{bernstein} are fulfilled for $Y_1,\dots,Y_m$ and we see that:
    \begin{subequations}
    \begin{gather}
        Pr\left(m\E{X} -\sum_{i=1}^m X_i\geq ma \right) \\ = Pr\left(\sum_{i=1}^m Y_i - m\E{Y} \geq ma \right) \\ \leq \exp{\left(-\frac{ma^2}{2(\frac{1}{m}\sum_{i=1}^m\E{X_i^2}+Ma/3)}\right)}
        \end{gather}
    \end{subequations}
    It therefore follows that
    \begin{subequations}
    \begin{align}
        Pr\left(\abs{\sum_{i=1}^m X_i - m\E{X}} \geq ma \right) \leq 2\exp{\left(-\frac{ma^2}{2(\varplace+Ma/3)}\right)}\notag
        \end{align}
    \end{subequations}
    \end{proof}
We are now ready to prove the core result of the section on convex optimization, the sample size derivation for $m_k$. We remind ourselves by restating Theorem \ref{size_convex}.
\setcounter{theorem}{3}

\begin{theorem}
   When using Algorithm~\ref{alg:cap} to estimate $\svk$ for an SGD algorithm $\A_\pi$ with a convex, $\beta-$smooth and $L-$Lipschitz loss function, then
    \begin{equation*}
        m_k \geq \brackets{\frac{\frac{32T^2L^4}{\beta^2k^2}+\frac{8\s TL^2}{k\beta}+4\s a/3}{a^2}}\ln{\frac{2n}{b}},
    \end{equation*}
    suffices for $\Pr(\abs{\svkh - \svk}\geq a)\leq b/n$, and by union bound, it suffices to estimate the Shapley value:
    \begin{equation*}  Pr\brackets{\abs{\hat{\sv}_i-\sv_i} \geq a}\leq b.
    \end{equation*}
 \begin{proof}

    Since each  $\ell(\cdot; z) \in [0, \s]$, we know that $\ldiff \leq 2\s$. It follows by a direct application of Theorem \ref{thm:stabconvex} to Lemma \ref{lem:LVar} that
    \begin{equation}
        \E{(\bigL-\bigL')^2} \leq \frac{4L^4}{k^2}\brackets{\sum_{t=1}^T\alpha_t}^2 +  \s\frac{2L^2}{k}\sum_{t=1}^T\alpha_t.
    \end{equation}
    Using the fact that Theorem \ref{thm:stabconvex} assumes $\alpha_t \leq 2/\beta$, we can finally bound
    \begin{equation}\label{eq:convexvarbound}
        \E{(\bigL-\bigL')^2} \leq T^2\frac{16L^4}{\beta^2k^2}+\s\frac{4TL^2}{k\beta}
    \end{equation}
     Because we have an upper bound on $\E{\brackets{\bigL-\bigL'}^2}$ given by Lemma \ref{lem:LVar}, we have an upper bound on $\varplace = \frac{1}{m_k}\sum_{i=1}^{m_k}\E{X_i^2} = \E{\brackets{\bigL-\bigL'}^2}$ due to the assumptions that the marginal contributions are independent and identically distributed. It therefore holds that $\varplace \leq T^2\frac{16L^4}{\beta^2k^2}+\s\frac{4TL^2}{k\beta}$:\begin{subequations}\label{m_kexpbound}
        \begin{gather}
            2\exp{\left(-\frac{m_ka^2}{2(\varplace+2\s a/3)}\right)} \\\leq 2\exp{\left(-\frac{m_ka^2}{2(T^2\frac{16L^4}{\beta^2k^2}+\s\frac{4TL^2}{k\beta}+2\s a/3)}\right)}.
        \end{gather}
    \end{subequations}
    To see why the above holds, consider the function
    \begin{equation}
    Pr\brackets{-\frac{a}{b}} = \frac{1}{e^{\frac{a}{b}}}.
    \end{equation}
    The larger $b$ is, the smaller the exponent of $e$. For $b \leq c$, we thus have
    \begin{equation}
        e^{\frac{a}{b}} \geq e^{\frac{a}{c}}\Rightarrow \frac{1}{e^{\frac{a}{b}}} \leq \frac{1}{e^{\frac{a}{c}}}
    \end{equation}
    As $\varplace$ is in the denominator of the left hand side of Inequality \ref{m_kexpbound}, overestimating $\varplace$ by using $T^2\frac{16L^4}{\beta^2k^2}+\s\frac{4TL^2}{k\beta}$ to calculate the risk of deviation hence leads to the probability of a large deviation to be overestimated, which is desirable for our purpose. Using Theorem \ref{bernstein}, we can state
    \begin{subequations}
    \begin{gather}
        Pr\bigg(\abs{\svkh-\svk} \geq \alpha\bigg) \\ = Pr\bigg(\abs{m_k\svkh-m_k\svk}\bigg) \geq m_k\alpha\bigg) \\\leq 2\exp{\bigg(-\frac{m_k\alpha^2/2}{\varplace+2G\alpha/3}\bigg)}
        \end{gather}
    \end{subequations}

    Since we want the probability of $\abs{\hat{\sv} - \sv} \geq a$ to be at most $b$, and we have $n$ layers we are sampling, in each layer, we want a risk of $\frac{b}{n}$ in each layer.

We now derive $m_k$ as follows:
    \begin{subequations}
    \begin{gather}
    Pr\bigg(\abs{\svkh-\svk} \geq \alpha\bigg) \\ \leq
        2\exp{\bigg(-\frac{m_ka^2}{2\varplace+4G a/3}\bigg)} \leq \frac{b}{n} \Rightarrow \\
        -\frac{m_ka^2}{2\varplace+4G a/3} \leq \ln{\frac{b}{2n}}\Rightarrow \\
      \frac{m_ka^2}{2\varplace+4G a/3} \geq \ln\frac{2n}{b} \Rightarrow\\
        m_k \geq \bigg(\frac{2\varplace+4G a/3}{ a^2}\bigg)\ln\frac{2n}{b}
    \end{gather}
\end{subequations}
Substituting $\varplace = \frac{16T^2L^4}{\beta^2k^2}+\frac{4\s TL^2}{k\beta}$ yields the stated result.
    \end{proof}
    \end{theorem}

\subsection{Proof of Theorem \ref{thm:t_0variance_varphi}}
\begin{theorem}
Suppose that, under the fixed subsequence interpretation, the number of samples required to compute $\svkh$ such that $Pr\brackets{\abs{\svkh-\svk} \geq \frac{a}{2}} \leq \frac{b}{2n}$ using an algorithm satisfying expected uniform stability with constant $\epsilon$ and loss function $\ell \in [0,\s]$ is $m_k$, then we can sample
    \begin{equation}\label{eq:singlecoalition}
        \vsample \geq \frac{8\brackets{\bound+\s a/3}}{a^2}\ln{\frac{4n m_k}{b}}
    \end{equation}
    permutations of each coalition $S \in \hat{S}_k$, where $\hat{S}_k$ is the set of $m_k$ coalitions sampled to compute $\svkh$ such that we can compute estimate $\svkho$ of $\svkh$ satisfying $Pr\brackets{\abs{\svkho-\svk} \geq a} \leq \frac{b}{n}$.
    \begin{proof}
        The sample size follows from an application of Bernstein's Inequality. For this purpose, we call $ \hat{\mathbb{S}}$ the set of $h$ sampled permutations taken from $\mathbb{S}$ and  observe that
        \begin{equation}
            \varplace = \frac{1}{h}\sum_{s \in \hat{\mathbb{S}}}\E{\brackets{ \bigL(\A(s))-\bigL(\A(s'))}^2}
        \end{equation}
        It hence follows from Lemma \ref{lem:LVar} that
        \begin{equation}
            \varplace \leq \bound
        \end{equation}
        We can substitute our values into Bernstein's inequality to obtain:
        \begin{subequations}
        \begin{gather}
            Pr\brackets{\abs{\hat{v}_i(S)-v_i(S)}\geq \frac{a}{2}} \\ \leq 2\exp{\brackets{-\frac{\vsample a^2}{8(\bound+2\s a/6)}}} \leq \frac{b}{2m_kn} \Rightarrow  \\
            \frac{\vsample a^2}{8(\bound+\s a/3)} \geq \ln{\frac{4n m_k}{b}} \Rightarrow \\
            \vsample \geq \frac{8\brackets{\bound+\s a/3}}{a^2}\ln{\frac{4n m_k}{b}}
            \end{gather}
        \end{subequations}
         By the union bound, the probability that for at least one coalition $S$, $\abs{\valhat - \val} \geq \frac{a}{2}$ holds is
    \begin{equation}
    Pr\brackets{\exists S \in \hat{S}: \abs{\valhat-\val} \geq \frac{a}{2}} \leq \sum_{j = 1}^{m_k} \frac{b}{2n m_k} = \frac{b}{2n}
    \end{equation}
    To see why our $(a,b)-$bound holds, we introduce $\svkho$ defined as
    \begin{equation}
        \svkho = \frac{1}{m_k}\sum_{S \in \hat{S}_k}\valhat
    \end{equation}
    Remember that $\svkh = \frac{1}{m_k}\sum_{S \in \hat{S}_k} \val$, thus $\svkho-\svkh = \frac{1}{m_k}\sum_{S \in \hat{S}_k} \brackets{\valhat - \val}$ and
    \begin{subequations}
    \begin{gather}
        \abs{\svkho - \svkh} = \abs{\frac{1}{m_k}\sum_{S \in \hat{S}_k} \brackets{\valhat - \val}} \\ \leq \frac{1}{m_k}\sum_{S \in \hat{S}_k} \abs{\valhat - \val}.
        \end{gather}
    \end{subequations}    Further, we observe that the event $\sum_{k=1}^n \abs{\valhat - \val} \geq m_k \frac{a}{2}$ requires that there exists at least one $S$ such that $\abs{\valhat-\val} \geq \frac{a}{2}$. We can therefore conclude that
    \begin{subequations}
    \begin{gather}
        Pr\brackets{\abs{\svkho-\svkh} \geq \frac{a}{2}} \\ \leq Pr\brackets{\frac{1}{m_k}\sum_{S \in \hat{S}_k} \abs{\valhat - \val} \geq \frac{a}{2}} \\= Pr\brackets{\sum_{S \in \hat{S}_k} \abs{\valhat - \val} \geq m_k \frac{a}{2}} \leq \frac{b}{2n}
        \end{gather}
    \end{subequations}
    We note that $Pr\brackets{\abs{\svkho -\svk} \geq a}$ is equivalent to
    \begin{equation}
    Pr\brackets{\brackets{\abs{\svkh-\svk}\geq \frac{a}{2}} \cup \brackets{\abs{\svkho-\svkh}\geq \frac{a}{2}}} \leq \frac{b}{n}
    \end{equation}
    which follows from the union bound.
    \end{proof}
\end{theorem}
\subsection{Proof of Corollary \ref{cor:repeated_sampling_convex}}
The claim follows from the fact that in each layer $k$, the number of coalitions to be sampled must be at least $m_k$ in order to obtain a $(a,b)-$bound on $\abs{\svkh-\svk}$. A minimum of $\frac{8\brackets{\bound+\s a/3}}{a^2}\ln{\frac{4n m_k}{b}}$ permutations must be sampled from each coalition, where $\bound = \frac{16T^2L^4}{\beta^2k^2}+\frac{4\s TL^2}{k\beta}$ for convex loss functions under the assumptions of Theorem \ref{thm:stabconvex}.
\subsection{Proof of Theorem \ref{Hardtbound}}
\begin{theorem}
Assume that $\ell(\cdot; z) \in [0, \s]$ is an L-Lipschitz and $\beta$-smooth loss function for
every z. Suppose that we run SGM for T steps with monotonically non-increasing step sizes $\alpha_t \leq c/t$. Then, for $w_T = A(S)$ and $w_T' = A(S')$, where $S,S'$ are neighbouring datasets, we have that:
    \begin{equation}
         \E{\abs{\ell(w_T;z)-\ell(w_T';z)}} \leq \s^{\frac{\beta c}{\beta c+1}}(2cL^2)^{\frac{1}{\beta c+1}}T^{\frac{\beta c}{\beta c
         +1}}\frac{1+\frac{1}{\beta c}}{k-1},
    \end{equation}
    where $k$ is the size of the dataset.
\begin{proof}
    It follows from the proof of Theorem 3.12 in \cite{hardt_train_2016} that
    \begin{equation}\label{HDef}
        \E{\abs{\ell(w_T;z)-\ell(w_T';z)}} \leq \frac{t_0}{k}\s + \frac{2L^2}{\beta(k-1)}\bigg(\frac{T}{t_0}\bigg)^{\beta c} = H_{sup}
    \end{equation}
    To find the minimum value of the right hand side of the equation above, we need to find $t_0$ such that:
    \begin{subequations}
    \begin{align}
        \frac{d}{dt_0}\frac{t_0}{k}\s + \frac{2L^2}{\beta(k-1)}\bigg(\frac{T}{t_0}\bigg)^{\beta c} = 0 \Rightarrow \\
        \frac{\s}{k} - \frac{2cL^2T^{\beta c}t_0^{-\beta_c -1}}{k-1} = 0
        \Rightarrow \\
        t^{\beta c+1} = \frac{2kcL^2T^{\beta c}}{\s(k-1)} \Rightarrow \\
        t_0 = \sqrt[\beta c+1]{\frac{k}{k-1}}\sqrt[\beta c+1]{\frac{2cL^2T^{\beta c}}{G}} \approx \sqrt[\beta c+1]{\frac{2cL^2T^{\beta c}}{G}}
        \end{align}
    \end{subequations}
    substituting this expression back into the right hand side of Equation \ref{HDef} and writing $q = \beta c$ for convenience, we obtain:
    \begin{subequations}
    \begin{gather}
        H_s = \frac{\beta (k-1)t_0^{q+1} \s + 2kL^2T^q}{\beta k (k-1)t_0^q} \Rightarrow \\
        H_s =
        \frac{\beta (k-1) 2cL^2T^{q}+2kL^2T^q}{\beta k (k-1)\bigg(\frac{2cL^2T^{q}}{G}\bigg)^\frac{q}{q+1}} \\=  \frac{2cL^2T^q\bigg(\beta (k-1)+c^{-1}k\bigg)}{\beta k (k-1)\bigg(\frac{2cL^2}{\s}\bigg)^{\frac{q}{q+1}}T^{\frac{q^2}{q+1}}}
        \end{gather}
    \end{subequations}
Simplifying further, we arrive at:
\begin{subequations}
    \begin{align}
        &H_s = \s^{\frac{q}{q+1}}\frac{(2cL^2)^{\frac{1}{q+1}}T^{\frac{q}{q+1}}\bigg(\beta (k-1)+\frac{k}{c}\bigg)}{\beta k (k-1)}\Rightarrow \\ &\leq \s^{\frac{q}{q+1}}(2cL^2)^{\frac{1}{q+1}}T^{\frac{q}{q+1}}\frac{\beta k(1+\frac{1}{\beta c})}{\beta k (k-1)} = \\
        &= \s^{\frac{q}{q+1}}(2cL^2)^{\frac{1}{q+1}}T^{\frac{q}{q+1}}\frac{1+\frac{1}{\beta c}}{k-1}
    \end{align}
    \end{subequations}

\end{proof}

\end{theorem}
\subsection{Proof of Theorem \ref{m_k size}}
\begin{theorem}
    In each layer $k$, let $m_k$ be the number of coalitions sampled and let $\svkh$ be the estimate of $\phi_k$. Then, if
    \begin{equation}
        m_k \geq 2\ln{\frac{2n}{b}}\brackets{\frac{2\hat{H}_k^2+2\s \hat{H}_k+4\s a/3}{a^2}},
    \end{equation}
    where
    \begin{equation}
        \hat{H}_k =\s^{\frac{\beta c}{\beta c+1}}(2cL^2)^{\frac{1}{\beta c+1}}T^{\frac{\beta c}{\beta c+1}}\frac{1+\frac{1}{\beta c}}{k-1}
    \end{equation}
    then
    \begin{equation}
        Pr\brackets{\abs{\hat{\sv}_i-\sv_i} \geq a}\leq b
    \end{equation}

    \begin{proof}
 As in the proof of Theorem \ref{size_convex}, we state
    \begin{subequations}
    \begin{gather}
        Pr\bigg(\abs{\svkh-\svk} \geq \alpha\bigg) \\= Pr\bigg(\abs{m_k\svkh-m_k\svk}\bigg) \geq m_k\alpha\bigg) \\ \leq 2\exp{\bigg(-\frac{m_k\alpha^2/2}{\varplace+2G\alpha/3}\bigg)}
        \end{gather}
    \end{subequations}
    for $\varplace = \sum_{i=1}^{m_k}\E{X_i^2}$. It follows from Lemma \ref{lem:LVar} that $\varplace \leq \hat{H}_k^2 + \s \hat{H}_k$ and:
        \begin{subequations}\label{m_kexpbound}
        \begin{gather}
        2\exp{\left(-\frac{m_ka^2}{2(\varplace+2\s a/3)}\right)} \\ \leq 2\exp{\left(-\frac{m_ka^2}{2(\hat{H}_k^2+\s \hat{H}_k+2\s a/3)}\right)}.
        \end{gather}
    \end{subequations}
     We derive $m_k$ as follows:
    \begin{subequations}
    \begin{gather}
    Pr\bigg(\abs{\svkh-\svk} \geq \alpha\bigg) \\ \leq
        2\exp{\bigg(-\frac{m_ka^2/2}{\hat{H}_k^2+\s \hat{H}_k+2G a/3}\bigg)} \leq \frac{b}{n} \Rightarrow \\
        -\frac{m_ka^2/2}{\hat{H}_k^2+\s \hat{H}_k+2G a/3} \leq \ln{\frac{b}{2n}}\Rightarrow \\
        \frac{m_k a^2}{2\hat{H}_k^2+2\s \hat{H}_k+4G a/3} \geq \ln\frac{2n}{b} \Rightarrow \\
        m_k \geq \bigg(\frac{2\hat{H}_k^2+2\s \hat{H}_k+4G a/3}{ a^2}\bigg)\ln\frac{2n}{b}
    \end{gather}
    \end{subequations}
The expression for $\hat{H}_k$ follows from Theorem \ref{Hardtbound}. Note that due to the $k-1$ term in the denominator of $\hat{H}_k$, we assume $\phi_i^0=\phi_i^1 = 0$ in this case.
    \end{proof}
    \end{theorem}
\begin{figure*}[hbt!]
\centering
\begin{tabular}{ccc}
\includegraphics[width=2in]{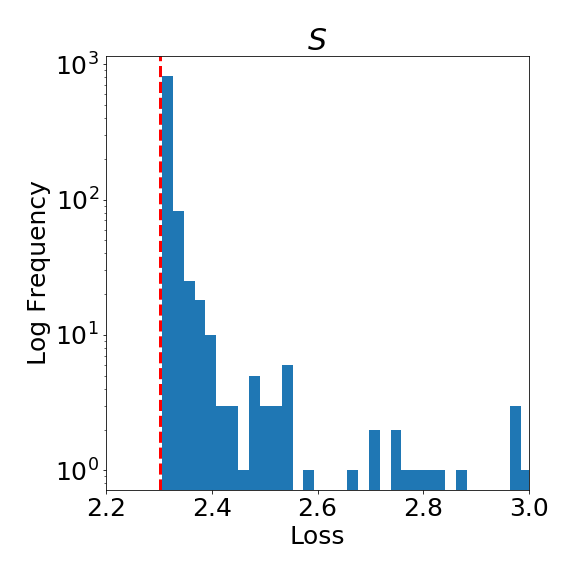} &
\includegraphics[width=2in]{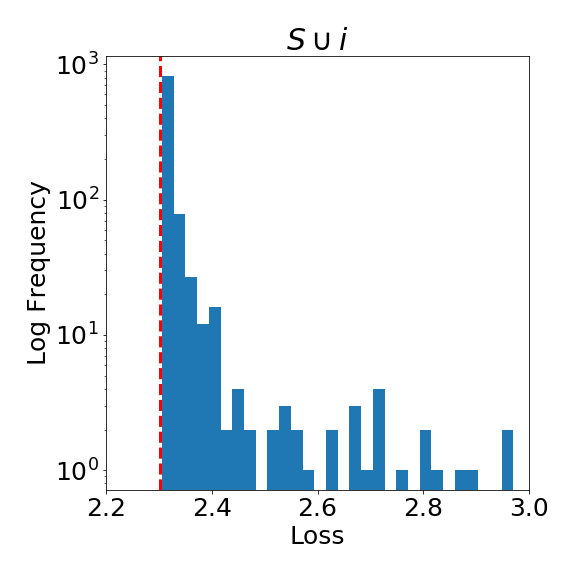}&\includegraphics[width=2in]{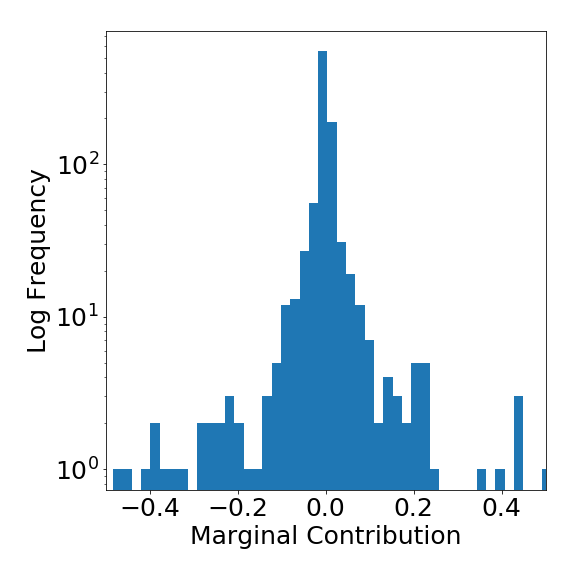}
\\
e) Losses (no $i$)  &f) Losses (with $i$) & (g) Marginal Contributions \\
\end{tabular}
 \caption{Distributions of losses and marginal contributions obtained for 1000 training runs of SGD with the same small training set $S$, where $|S|$=10. The red vertical line in plots a) and b) shows a loss value equivalent to randomly guessing one of the 10 classes, therefore corresponding to a baseline random loss value.}
\label{fig:small-losses}
\end{figure*}

\section{Further Experimental Details}
Strongly convex experiments were completed using a 2.9 GHz Quad-Core Intel Core i7 processor. Non-convex experiments used NVIDIA RT2080Ti 11GB GPUs.
The deterministic variant of CUDNN was used with benchmarking disabled for Pytorch models.

Logistic regression models were trained using Scikit-Learn's~\cite{pedregosa2011scikit}  Liblinear solver. The CNN architecture had 1 convolutional layer and 2 fully connected layers, with RELU activations and max-pooling.  All CNN models were trained via stochastic gradient descent (batch size $1$) with a learning rate of $0.01$ for $20$ or $30$ epochs. The pre-trained CNN was trained for $5$ epochs with learning rate $0.01$ on a disjoint training set of size $1000$ from the same binary-Fashionmnist task.

Due to computational constraints datasizes were reduced from the full available datasize to smaller subsets ($n=100$) that still resulted in significantly better performance than random test accuracy e.g. 0.73 and 0.84 for the Adult and Digits tasks respectively.

\section{Variation in Marginal Contributions for Small Coalitions}

As discussed in the main body of results, CNNs may fail to train meaningfully for very small datasizes such as those used by layer sampling strategies that focus exclusively on very small values of $k$. Figure~\ref{fig:small-losses} shows the distribution of the losses and marginal contributions obtained by performing 1000 training runs of SGD on a CNN using the same hyperparameters and the same small subset of data $S$, where $|S|=10$. As we can see in Figure~\ref{fig:small-losses} a) and b), the loss values of both $S$ and $S \cup i $ fall exclusively above the loss obtained by random guessing, suggesting random variation due to the initialization and data shuffling but no meaningful learning. The marginal contributions are approximately centered around 0 with a mean of $-3\times 10^{-3}$ however have standard deviation of $0.14$. This variation is high when considering that the final Shapley values for the same task are often in the range of $10^{-3}$ to $10^{-4}$.

\end{document}